\documentclass[letterpaper, 10 pt, conference]{ieeeconf}  

\IEEEoverridecommandlockouts                              
\usepackage{graphics} 
\usepackage{epsfig} 
\usepackage{times} 
\usepackage{amsmath}
\usepackage{amssymb}  
\usepackage{algorithm}
\usepackage{subcaption}
\usepackage{textcomp}
\usepackage{xcolor}
\usepackage{multirow}
\let\labelindent\relax
\usepackage{enumitem}
\usepackage{float}
\usepackage{cite}
\usepackage{booktabs}
\usepackage{stackengine}
\usepackage{etoolbox}
\usepackage{mathtools}
\usepackage{optidef}

\newtheorem{theorem}{Theorem}

\title{\LARGE \bf
Geometric Look-Angle Shaping Strategy for Enclosed Inspection
}

\author{Amit Shivam$^1$, Manuel C.R.M. Fernandes$^1$, Sérgio Vinha$^1$, and Fernando A.C.C. Fontes$^1$
\thanks{*This work is financially supported by Project 2022.02801.PTDC-UPWIND-ATOL (https://doi.org/10.54499/2022.02801.PTDC), as well as by grants 2021.06313.BD and 2021.07346.BD, and by the SYSTEC – Research Centre for Systems and Technologies (UID/00147) and the Associate Laboratory ARISE – Advanced Production and Intelligent Systems (LA/P/0112/2020, DOI: 10.54499/LA/P/0112/2020), all funded by Fundação para a Ciência e a Tecnologia, I.P./MECI through national funds.}
\thanks{$^1$Amit Shivam, Manuel C.R.M. Fernandes, Sérgio Vinha and Fernando A.C.C. Fontes are with SYSTEC-ISR ARISE, Faculdade de Engenharia, Universidade do Porto, Rua Dr. Roberto Frias, 4200-465 Porto, Portugal, Emails:
		{\tt\small amit@upwind.pt}, {\tt\small svinha@fe.up.pt}, {\tt\small mcrmf@fe.up.pt} 
        and {\tt\small faf@fe.up.pt}}
}

\begin{document}

\maketitle
\thispagestyle{empty}
\pagestyle{empty}

\begin{abstract}


This paper introduces inspection through GLASS, a Geometric Look-Angle Shaping Strategy for enclosed regions using unmanned aerial vehicles. In doing so, the vehicle’s guidance command is constructed through a bounded, geometry-consistent shaping of the look angle relative to a desired standoff path. By embedding a smooth, hyperbolic-tangent-type shaping function within a polar geometric framework, GLASS ensures global existence of the guidance dynamics. It avoids the far-field limitations inherent to conventional formulations.
Lyapunov stability analysis establishes asymptotic convergence to a prescribed inspection standoff under explicit curvature feasibility conditions, along with analytical settling-time characteristics. The proposed strategy incorporates maximum turn-rate constraints without inducing singularities throughout the workspace.
High-fidelity six-degree-of-freedom quadrotor simulations demonstrate the effectiveness of GLASS in representative enclosed inspection scenarios, highlighting a practically viable guidance framework for autonomous enclosed inspection missions.
\end{abstract}
 \section{Introduction}
Applications of unmanned aerial vehicles (UAVs) in persistent surveillance, reconnaissance, and cooperative monitoring missions have witnessed significant growth over the past decade. In such missions, a UAV is often required to regulate its motion relative to a target so as to maintain a prescribed radial separation, commonly referred to as the stand-off distance, while continuously circling the target. This task demands precise control of the relative geometry between the UAV and the target, particularly under the kinematic and curvature constraints inherent to fixed-wing platforms. Moreover, practical mission requirements impose time-critical performance objectives, necessitating guidance strategies that are computationally efficient, structurally simple, and capable of providing deterministic convergence guarantees. Designing such guidance laws—capable of ensuring reliable orbit acquisition within a bounded time while respecting vehicle turn-rate limitations—remains a fundamental challenge in autonomous standoff tracking.  

In this regard, several methods have been proposed in the literature, which can be broadly categorized as virtual-target-based methods \cite{park2004,ratnoo2015path}, relative side-bearing angle pursuit \cite{park2016circling}, bifurcation-based approaches \cite{anjaly2020target,srinivasu2022standoff,RaviScitech2025,RaviACC2025}, and vector-field-based guidance \cite{frew2008coordinated,lim2013standoff,pothen2017curvature,harinarayana2022vector,amitRatnoo2021,shivamRatnoo2023,amit2022vector,shivamLC2024}. Park et al.\cite{park2004} employed pursuit guidance logic to steer the UAV toward a moving reference target defined on the desired orbit. Ratnoo at al. \cite{ratnoo2015path} presents
trajectory shaping guidance law for following curved paths while overcoming the limitations
of guidance logic used in\cite{park2004}. Although simple in structure, their convergence properties depend on virtual-target placement and may degrade under large initial separation. Ref. \cite{park2016circling} regulates angular geometry directly and ensures circumnavigation along the desired orbit. However, the resulting transients are highly gain-sensitive and can induce large turn-rate demand during far-field capture. 
Bifurcation-based approaches have also emerged as a powerful framework for standoff tracking and inspection missions. Anjaly and Ratnoo\cite{anjaly2020target} employed pitchfork bifurcation in UAV look-angle dynamics to achieve controlled circumnavigation and target prosecution, wherein appropriate selection of the bifurcation parameter enables the desired motion primitive. Subsequently, a transcritical bifurcation in line-of-sight distance dynamics was introduced to achieve stable standoff tracking with tunable convergence behavior\cite{srinivasu2022standoff}. More recently, saddle-node infinite-period bifurcation has been utilized to guide quadrotors toward prescribed viewpoints on circular\cite{RaviScitech2025} and elliptic\cite{RaviACC2025} inspection boundaries, where angular convergence is shaped through a single bifurcation parameter while radial convergence is regulated independently. These methods fundamentally reshape the phase portrait of the engagement dynamics via a tunable parameter, enabling structured transitions between motion objectives. Nevertheless, the convergence rate and performance guarantees remain implicit in the parameters involved in the design criteria.

Vector-field-based methods provide a globally defined heading direction at every point in the workspace and admit Lyapunov-based stability guarantees.
Frew et al \cite{frew2008coordinated} developed classical Lyapunov vector fields that combine contraction and circulation components to ensure asymptotic convergence. Lim et al. \cite{lim2013standoff} introduced an iteratively computed guidance parameter to achieve simultaneous arrival of multiple UAVs on the circular orbit. Further, Pothen and Ratnoo \cite{pothen2017curvature} reshape the circulation component to modify the Lyapunov vector field method, resulting in faster settling time characteristics with lesser maximum curvature in comparison to \cite{frew2008coordinated}. Harinarayana et al.\cite{harinarayana2022vector} further extended vector-field constructions to improve transient behavior under curvature constraints by modifying the tangential and normal component synthesis.  Shivam and Ratnoo\cite{amitRatnoo2021,shivamRatnoo2023} employed an arcsine-based shaping function to construct a commanded course profile that defines a vector field over the plane, enabling global convergence to the desired circular standoff path.  This framework was subsequently extended to elliptic path-following \cite{amit2022vector} scenarios through a nonlinear parametrization of concentric level sets, allowing vector-field construction around general elliptic boundaries with Lyapunov-based stability guarantees. More recently, Lamé-curve-based vector-field guidance\cite{shivamLC2024} was introduced for rectangular boundary surveillance, where a continuous-curvature circumscribing path is first optimized under vehicle turn-rate constraints, and an arcsine-shaped heading command is then used to ensure asymptotic convergence to the Lamé path. While these vector-field approaches provide globally defined heading commands and analytical stability guarantees, the resulting engagement dynamics remain tightly coupled to the specific shaping function and geometric embedding of the desired path, often leading to implicit trade-offs between radial convergence speed, curvature demand, and transient turn-rate peaks. 

Recently, arcsine-based look-angle shaping\cite{jha2024standoff} for standoff tracking 
introduced a normalized error formulation that yields a closed-form settling-time expression with explicit gain bounds. The framework provides deterministic convergence guarantees and reduced control effort relative to bifurcation-based approaches. However, the admissible guidance gain is intrinsically coupled to the initial normalized error to preserve real-valued look-angle dynamics. As the initial separation increases, the allowable gain shrinks, thereby limiting achievable radial contraction rate and slowing far-field capture under strict turn-rate constraints. Consequently, the convergence speed becomes fundamentally constrained by initial conditions. This clearly indicates a structural limitation common to gain-amplified shaping strategies: faster convergence is achieved only through gain escalation, which is bounded by feasibility constraints embedded in the look-angle dynamics.

Motivated by this gap, the present work develops a geometry-driven look-angle shaping framework in which the approach geometry itself is treated as the primary design variable. Instead of amplifying the error through gain scaling, the proposed method prescribes a bounded nonlinear evolution of the look angle that intrinsically regulates radial contraction and curvature demand. Constructed in polar coordinates with explicit turn-rate feasibility embedded in the design, the resulting closed-loop dynamics admit global well-posedness and analytically tractable settling-time characterization under radial-dominant conditions. This enables fast far-field capture without sacrificing curvature feasibility, thereby enlarging the admissible operational envelope for practical standoff missions.

\section{Problem Formulation}

Consider a problem scenario as shown in Fig. \ref{fig: Problem scenario}, where a UAV is required to inspect a planar enclosed region from a standoff distance $r_d$. The centre of the region is assumed to be located at the origin of an inertial frame.
\begin{figure}[ht]
    \centering
    \includegraphics[width=0.75\linewidth]{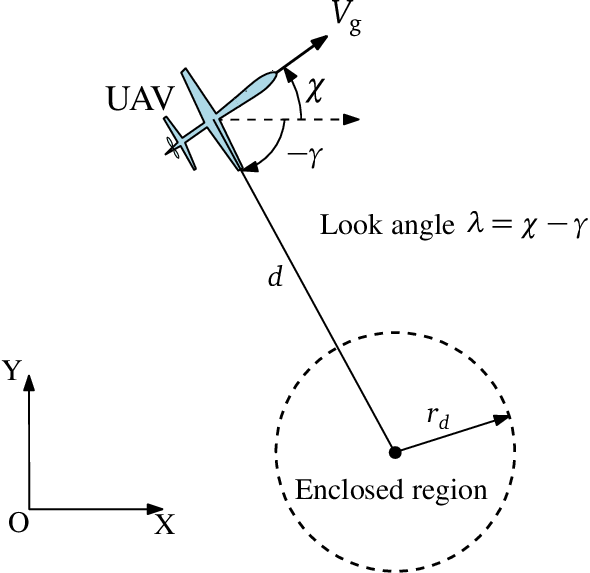}
    \caption{Problem scenario}
    \label{fig: Problem scenario}
\end{figure}
     For a planar inspection geometry, $d$ and $\gamma$ denote the range and the line-of-sight angle (LOS) on the plane of the UAV's trajectory relative to the origin.
The look angle $\lambda$ is the difference between the instantaneous course and the line-of-sight angle, which is expressed as
\begin{equation}
\lambda = \chi - \gamma.
\  label{eq: look angle def}
\end{equation}
The relative engagement kinematics of the UAV in polar coordinates are described by
\begin{align}
    \begin{split}
   \dot d &= V_g \cos\lambda \\
\dot\gamma &= \frac{V_g}{d} \sin\lambda.
\label{eq:polar_kinematics}     
    \end{split}
\end{align}
Let the desired standoff path be specified in polar form as a smooth closed curve
\begin{equation}
d = r(\gamma),
\end{equation}
where $r(\gamma)$ corresponds to a circle, ellipse, or Lam\'e (superellipse) curve.
The objective of this work is to design a guidance law for $\chi$ such that the standoff error reduces to zero and satisfies
\begin{enumerate}
    \item The settling time to a prescribed tube standoff error $e = d - r(\gamma)$ such that $|e|\le \varepsilon$ is analytically tractable.
    \item The commanded course rate satisfies the actuator constraint
    \begin{equation}
    |\dot\chi| \le \omega_{\max},
    \label{eq:yaw_constraint}
    \end{equation}
    where $\omega_{\max}$ is the maximum achievable yaw (or course) rate of the UAV.
\end{enumerate}

\section{Proposed Look-Angle Based Guidance Law}
This section discusses the geometry-driven proposed guidance law while emphasising the limitations of arcsine-based guidance law.

\subsection{Far-Field Limitation of Arcsine Shaping Function}

Look-angle shaping guidance laws used in\cite{jha2024standoff} impose an intrinsic feasibility constraint on the shaping gain. For a representative arcsine construction of the form
\begin{equation}
\lambda_a(e)=\frac{\pi}{2}+\sin^{-1}(k_D e), \quad  |k_D e(t)|<1
\end{equation}
A necessary feasibility condition at $t=0$ is  $|k_D e(t)|<1$, i.e., $k_D < \frac{1}{|e_0|}$. For far-field cases this forces $k_D$
 to be small, and therefore limits the achievable initial radial contraction.

\subsection{Proposed GLASS Shaping Law}

Motivated by the above limitation, we propose a bounded hyperbolic-tangent shaping guidance as
\begin{equation}
\sigma(e) = -\tanh(k_G e), \qquad k_G>0,
\end{equation}
embedded implicitly within the look-angle equation. Since $|\tanh(k_Ge)|<1$ for all $e\in\mathbb{R}$, the shaping function is globally defined and smooth for any combination of $k_G$ and $e$, with uniformly bounded slope
\begin{equation}
\left|\frac{d\sigma}{de}\right| = k_G\,\text{sech}^2(k_Ge) \le k_G.
\label{eq:sigma_slope}
\end{equation}
The highlight of this formulation is that the hyperbolic-tangent shaping guidance does not suffer from the limitation imposed by the initial conditions. 
Let the desired standoff curve be given in polar form $d=r(\gamma)$, and define the geometry coupling
\begin{equation}
a(\gamma,d)=\frac{r'(\gamma)}{d}.
\end{equation}
The commanded look angle $\lambda$ is selected to satisfy the implicit GLASS constraint
\begin{equation}
\cos\lambda - a(\gamma,d)\sin\lambda = -\sigma(e).
\end{equation}

\subsection{Geometric Feasibility and Regularity}

\begin{theorem}[Global feasibility]
 Define
\begin{equation}
\alpha := \tan^{-1}(a),\qquad
\gamma_\lambda(e) := \cos^{-1}\!\left(-\frac{\sigma(e)}{\sqrt{1+a^2}}\right)\in[0,\pi],
\label{eq:def_phi_theta}
\end{equation}. For any $e\in\mathbb{R}$, the algebraic constraint
\begin{equation}
\cos\lambda - a\sin\lambda = -\sigma(e)
\label{eq: GLASS condition}
\end{equation}
admits a real solution $\lambda(e,\gamma,d)$
and the two solution branches
\begin{equation}
\lambda_d(e) = -\phi + s\,\gamma_\lambda(e),\qquad s\in\{+1,-1\}.
\label{eq:lambda_branches}
\end{equation}
Then the following hold:
\begin{enumerate}
\item \textbf{Existence for all $e\in\mathbb{R}$.} The quantity $\gamma_\lambda(e)$ is well-defined for all $e\in\mathbb{R}$, i.e.,
\begin{equation}
-\frac{\sigma(e)}{\sqrt{1+a^2}} \in (-1,1)\quad \forall e\in\mathbb{R},
\end{equation}
and hence $\lambda_d(e)$ exists globally without any initial-condition-dependent feasibility constraint.

\item \textbf{Boundedness and smoothness.} Each branch $\lambda_d(e)$ is continuous for all $e\in\mathbb{R}$ and is continuously differentiable with bounded slope. In particular, for all $e\in\mathbb{R}$,
\begin{equation}
\bigg|\frac{d\lambda_d}{de}\bigg|
=
\bigg|\frac{d\gamma_\lambda}{de}\bigg|
\le
\frac{k_G}{\sqrt{1+a^2}}.
\label{eq:lambda_slope_bound}
\end{equation}

\end{enumerate}

Moreover, each solution branch is continuous and continuously differentiable in $e$, with globally bounded slope.
\end{theorem}

\begin{proof}
Using the phase-shift identity
\begin{equation}
\cos\lambda - a\sin\lambda
= \sqrt{1+a^2}\cos(\lambda+\alpha), \qquad \alpha=\tan^{-1}(a),
\end{equation}
the constraint becomes
\begin{equation}
\cos(\lambda+\alpha)=\frac{-\sigma(e)}{\sqrt{1+a^2}}.
\end{equation}
Because $|\sigma(e)|<1$ and $\sqrt{1+a^2}\ge1$, the argument of $\cos^{-1}(\cdot)$ always lies strictly in $(-1,1)$. Hence, a real solution exists for all $e\in\mathbb{R}$, independent of initial conditions.
Differentiating \eqref{eq:def_phi_theta} gives
\begin{equation}
\frac{d\gamma_\lambda}{de}
=
-\frac{1}{\sqrt{1-\left(\sigma(e)/\sqrt{1+a^2}\right)^2}}
\cdot
\frac{1}{\sqrt{1+a^2}}
\cdot
\frac{d\sigma}{de}.
\label{eq:dtheta_de}
\end{equation}
Because $|\sigma(e)|<1$ and $\sqrt{1+a^2}\ge 1$, we have
\begin{align}
\begin{split}
 0< 1-\left(\frac{\sigma(e)}{\sqrt{1+a^2}}\right)^2 \le 1 \\
\Rightarrow \quad
\frac{1}{\sqrt{1-\left(\sigma(e)/\sqrt{1+a^2}\right)^2}} \le \frac{1}{\sqrt{1-0}} = 1,   
\end{split} 
\end{align}
and therefore, combining with \eqref{eq:sigma_slope} in \eqref{eq:dtheta_de} yields the global slope bound
\begin{equation}
\bigg|\frac{d\gamma_\lambda}{de}\bigg|
\le
\frac{1}{\sqrt{1+a^2}}\,\bigg|\frac{d\sigma}{de}\bigg|
\le
\frac{k}{\sqrt{1+a^2}}.
\end{equation}
Since $\lambda_d(e)=-\phi+s\gamma_\lambda(e)$, the same bound holds for $d\lambda_d/de$, proving \eqref{eq:lambda_slope_bound}.
Since $\sigma(e)$ is smooth with bounded derivative, and the denominator in the implicit derivative expression never vanishes, the slope $d\lambda/de$ remains globally bounded. Therefore, no arcsine-type singularity or initial-condition-dependent feasibility constraint arises.
\end{proof}

\subsection{Closed-Loop Stability on General Standoff Curves}

Define the scalar standoff error
\begin{equation}
e(t)=d(t)-r(\gamma(t)).
\end{equation}
Differentiating and substituting the polar engagement kinematics yields
\begin{equation}
\dot e = V_g\left(\cos\lambda - a(\gamma,d)\sin\lambda\right).
\end{equation}
Imposing the GLASS constraint \eqref{eq: GLASS condition} gives the reduced scalar dynamics
\begin{equation}
\dot e = -V_g\tanh(k_G e).
\label{eq:e_dyn_tanh}
\end{equation}
In the ideal guidance setting, we assume the commanded look angle 
$\lambda$ satisfies the algebraic GLASS constraint at all times, yielding the reduced scalar error dynamics below.
\begin{theorem}[Global asymptotic convergence]
Consider the scalar dynamics \eqref{eq:e_dyn_tanh}, then
\begin{enumerate}
\item Solutions exist uniquely for all $t\ge0$.
\item The equilibrium $e=0$ is globally asymptotically stable.
\item The error evolves monotonically toward zero without overshoot.
\end{enumerate}
\label{thm: theorem 2}
\end{theorem}

\begin{proof}
The right-hand side is globally Lipschitz because
\begin{equation}
\left|\frac{d}{de}\left(-V_g\tanh(k_Ge)\right)\right|
= V_g k_G\,\text{sech}^2(k_Ge) \le V_g k_G.
\end{equation}
Hence, solutions exist uniquely for all time.
Consider the Lyapunov function
\begin{equation}
V(e)=\int_0^e \tanh(k_G s)\,ds
= \frac{1}{k_G}\ln\left(\cosh(k_G e)\right).
\end{equation}
Since $\cosh(\cdot)\ge1$, we have positive definite $V(e)\ge0$ with equality only at $e=0$. Along trajectories,
\begin{equation}
\dot V = \tanh(k_G e)\dot e
= -V_g\tanh^2(k_G e)\le0,
\end{equation}
and $\dot V=0$ only when $e=0$. Thus, $e=0$ is globally asymptotically stable.
Furthermore, if $e>0$ then $\dot e<0$, and if $e<0$ then $\dot e>0$. Hence the solution cannot change sign and converges monotonically to zero.
\end{proof}

\subsection{Analytical Tube-Entry Time}




\begin{theorem}
\label{thm:settling_time}
Under the hypotheses of Theorem~\ref{thm: theorem 2}, let $e(t)$ evolve according
to \eqref{eq:e_dyn_tanh} with initial error $e(0)=e_0$.
For a prescribed tube size $\varepsilon>0$, define the tube-entry time
\begin{equation}
T_{\varepsilon}\triangleq \inf\{t\ge 0:\ |e(t)|\le \varepsilon\}.
\label{eq:T_eps_def}
\end{equation}
If $|e_0|>\varepsilon$, then the tube-entry time is given in closed form by

\begin{equation}
T_{\varepsilon}
= \frac{1}{k_G V_g}\,
\ln\!\Bigg(\frac{\sinh\!\left(k_G|e_0|\right)}{\sinh\!\left(k_G\varepsilon\right)}\Bigg).
\label{eq:T_eps_closed_form}
\end{equation}

If $|e_0|\le \varepsilon$, then $T_{\varepsilon}=0$.
\end{theorem}

\begin{proof}
From \eqref{eq:e_dyn_tanh}, the dynamics are odd and sign-preserving; thus it
suffices to consider $e_0>0$ (the case $e_0<0$ follows by symmetry and yields
the same expression in terms of $|e_0|$).
For $e>0$, \eqref{eq:e_dyn_tanh} implies $\dot e<0$, so $e(t)$ decreases
monotonically until it reaches $\varepsilon$.
Separation of variables for $e(t)\in[\varepsilon,e_0]$ yields
\begin{equation}
\frac{de}{\tanh(k_Ge)} = -V_g\,dt.
\label{eq:sep_vars}
\end{equation}
Integrating from $t=0$ to $t=T_{\varepsilon}$ and from $e(0)=e_0$ to
$e(T_{\varepsilon})=\varepsilon$ gives
\begin{equation}
\int_{e_0}^{\varepsilon}\frac{de}{\tanh(k_Ge)}
= -V_g\int_{0}^{T_{\varepsilon}}dt
= -V_g T_{\varepsilon}.
\label{eq:int_eq}
\end{equation}
Using the identity (valid for $e>0$)
\[
\int \frac{de}{\tanh(k_Ge)}
= \frac{1}{k_G}\ln\left(\sinh(k_Ge)\right)+C,
\]
we obtain
\begin{align}
\int_{e_0}^{\varepsilon}\frac{de}{\tanh(k_Ge)}
&= \frac{1}{k_G}\left[\ln\sinh(k_G\varepsilon)-\ln\sinh(k_Ge_0)\right] \nonumber\\
&= -\frac{1}{k_G}\ln\left(\frac{\sinh(k_Ge_0)}{\sinh(k_G\varepsilon)}\right).
\end{align}
Substituting into \eqref{eq:int_eq} yields
\[
-V_g T_{\varepsilon}
= -\frac{1}{k_G}\ln\left(\frac{\sinh(k_Ge_0)}{\sinh(k_G\varepsilon)}\right),
\]
and hence
\[
T_{\varepsilon}
= \frac{1}{k_G V_g}\ln\left(\frac{\sinh(k_Ge_0)}{\sinh(k_G\varepsilon)}\right).
\]
Replacing $e_0$ by $|e_0|$ to cover both signs gives \eqref{eq:T_eps_closed_form}.
If $|e_0|\le\varepsilon$, then the initial condition is already inside the tube,
so $T_{\varepsilon}=0$ by definition \eqref{eq:T_eps_def}.
\end{proof}

\section{Simulation results}
This section presents numerical simulation studies to evaluate the performance of the inspection strategy using the proposed GLASS-based guidance scheme. We first consider ideal guidance dynamics to examine the convergence properties of the geometric look-angle shaping mechanism. The analysis is then extended to a first-order course-control while satisfying turn-rate constraints. Finally, comparative simulations are conducted against representative guidance laws from the existing literature to highlight the advantages of the proposed approach. Unless stated otherwise, $V_g = 20$ m/s, a desired circular orbit with centre at the origin and radius $r_d = 200$ m, standoff settling tube $\epsilon = 0.05$ is considered.

\subsection{Ideal dynamics}
In the first case, numerical simulations are performed under ideal guidance dynamics using the polar kinematic model \eqref{eq:polar_kinematics}, with $\chi$ set to the commanded course angle $ \chi_\mathrm{d}$, as derived from the GLASS guidance law. Table \ref{tab:ideal_manifold_IC} lists the initial conditions used for simulation for two cases: initial position outside and inside the desired orbit.

\begin{table}[ht]
\centering
\caption{Initial conditions for ideal dynamics}
\label{tab:ideal_manifold_IC}

\begin{tabular}{l c c c}
\toprule
Parameter                  & Symbol          & \multicolumn{2}{c}{Value} \\
\cmidrule(lr){3-4}
                           &                 & Outside       & Inside \\
\midrule
Initial position           & $(x_0, y_0)$    & (450, -250) m & (45, -25) m \\
Initial range              & $d_0$           & 514.78 m      & 51.48 m \\
Initial standoff error     & $e_0$           & 314.78$^\circ$ & -148.52 m \\
Initial LOS angle          & $\gamma_{\mathrm{0}}$ & -29.05$^\circ$ & -29.05$^\circ$ \\
\bottomrule
\end{tabular}
\end{table}


Fig. \ref{fig: Results for ideal dynamics for outside initial position} illustrates the simulation results for outside initial positions. Fig. \ref{fig: UAV trajectories outside ideal} plots all UAV trajectories that converge smoothly toward the desired circular orbit. The corresponding standoff error profiles reducing to negligibly low value are shown in Fig. \ref{fig: Error profiles outside ideal}. As indicated in Fig. \ref{fig: Error profiles outside ideal}, increasing $k_G$ values makes the radial error reaching to the prescribed $\epsilon$
tube, with larger shaping gains yielding shorter tube-entry times. Fig. \ref{fig: Turn rate variations outside ideal} depicts the turn-rate profiles associated with the convergence maneuver. Larger $k_G$
generate higher turn-rate peaks, resulting in large maximum curvature, as evident in Fig. \ref{fig: Curvature with standoff error outside ideal}.
\begin{figure}
    \begin{subfigure}[b]{.25\textwidth}
			\centering			\includegraphics[width=\linewidth]{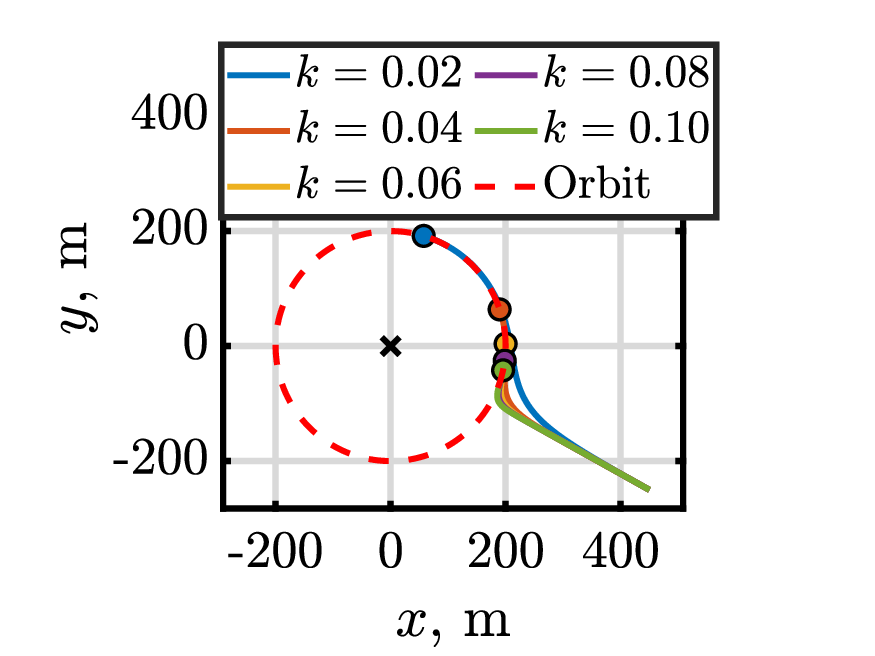}
			\caption{ UAV trajectories}
			\label{fig: UAV trajectories outside ideal}     
		\end{subfigure}%
		\begin{subfigure}[b]{.25\textwidth}
			\centering		\includegraphics[width=\linewidth]{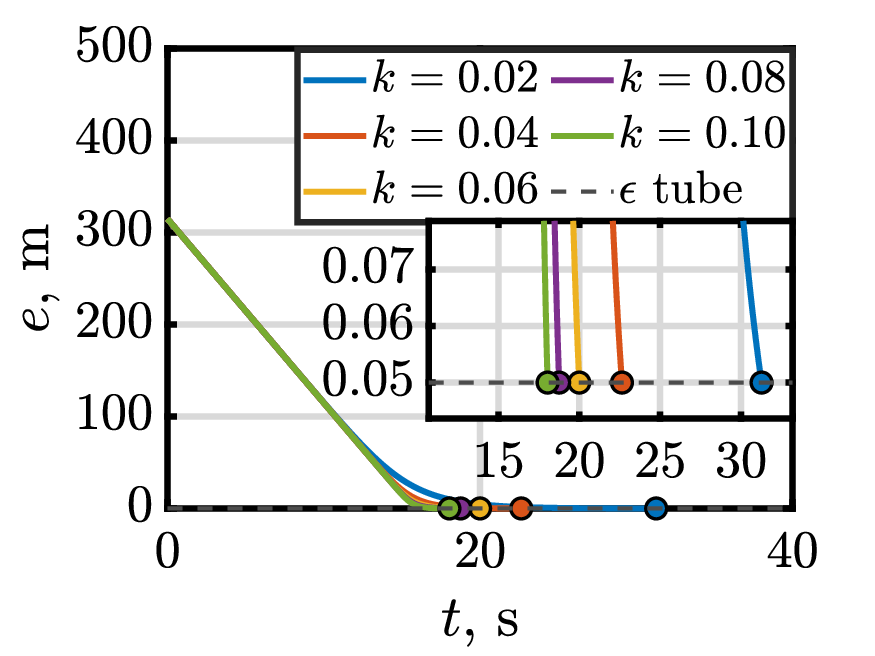}
			\caption{ Standoff error profiles}
			\label{fig: Error profiles outside ideal}
		\end{subfigure}%
        \qquad
       \begin{subfigure}[b]{.25\textwidth}
			\centering			\includegraphics[width=\linewidth]{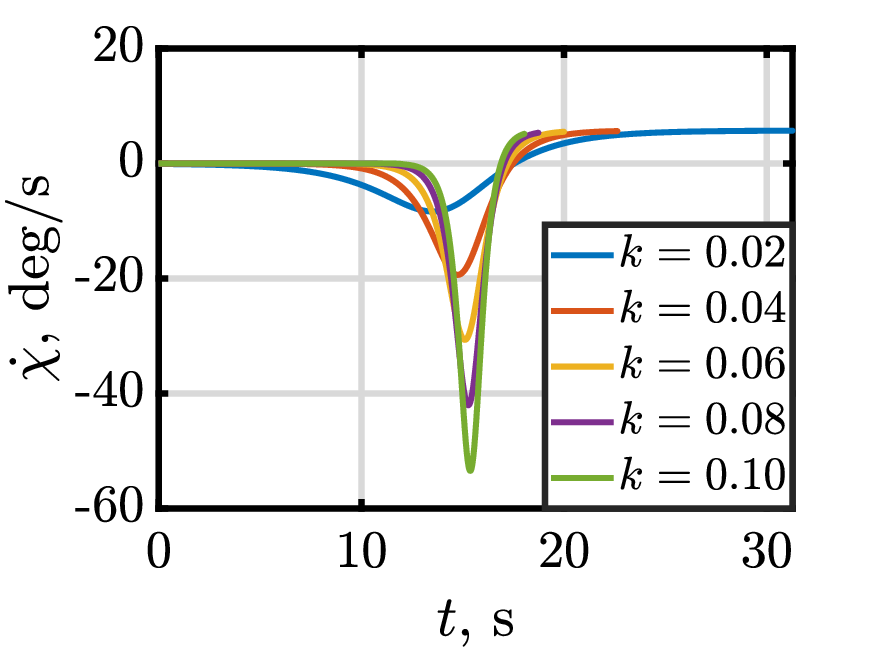}
			\caption{ Turn rate variations}
			\label{fig: Turn rate variations outside ideal}     
		\end{subfigure}%
		\begin{subfigure}[b]{.25\textwidth}
			\centering		\includegraphics[width=\linewidth]{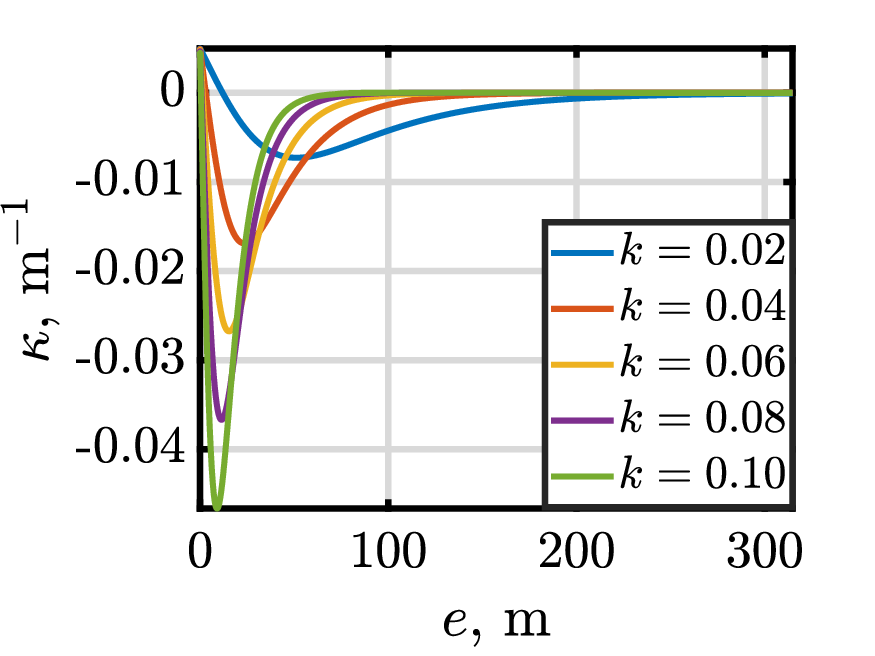}
			\caption{ Curvature with standoff error }
			\label{fig: Curvature with standoff error outside ideal}
		\end{subfigure}%
        \caption{ Results for ideal dynamics for outside initial position}
        \label{fig: Results for ideal dynamics for outside initial position}
        \end{figure}
        Similarly, Fig. \ref{fig: Results for ideal dynamics for inside initial position} plots numerical results for inside initial positions. UAV trajectories are plotted in Fig. \ref{fig: UAV trajectories inside ideal} with the corresponding standoff error profiles reducing to zero is shown in Fig. \ref{fig: Error profiles inside ideal}. The turn-rate profiles are shown in Fig. \ref{fig: Turn rate variations inside ideal}, where increasing $k_G$ results in turn rate peaking behaviour. Notably, the curvature–standoff error relationship in Fig. \ref{fig: Curvature with standoff error inside ideal} remains smooth and bounded across all cases. For both outside and inside initial conditions, the numerically obtained settling times closely match the analytical predictions for the ideal guidance dynamics, as listed in Table \ref{tab:ideal_settling_times}.
        \begin{table}[t]
\centering
\caption{Closed-form analytic settling time $T_{\mathrm{s}}$ for ideal-manifold dynamics
under inside and outside initial conditions for different gain values $k_G$.}
\label{tab:ideal_settling_times}
\setlength{\tabcolsep}{3pt}
\begin{tabular}{c|ccccc}
\hline
Standoff 
& $k_G=0.02$ 
& $k_G=0.04$ 
& $k_G=0.06$ 
& $k_G=0.08$ 
& $k_G=0.10$ \\ 
\hline
Inside    
& \multirow{2}{*}{22.956} 
& \multirow{2}{*}{14.328} 
& \multirow{2}{*}{11.689} 
& \multirow{2}{*}{10.444} 
& \multirow{2}{*}{9.729} \\
$(e_0<0)$ & & & & & \\
Outside  
& \multirow{2}{*}{31.276} 
& \multirow{2}{*}{22.641} 
& \multirow{2}{*}{20.002} 
& \multirow{2}{*}{18.757} 
& \multirow{2}{*}{18.042} \\ 
$(e_0>0)$ & & & & & \\
\hline
\end{tabular}
\end{table}

        \begin{figure}
    \begin{subfigure}[b]{.25\textwidth}
			\centering			\includegraphics[width=\linewidth]{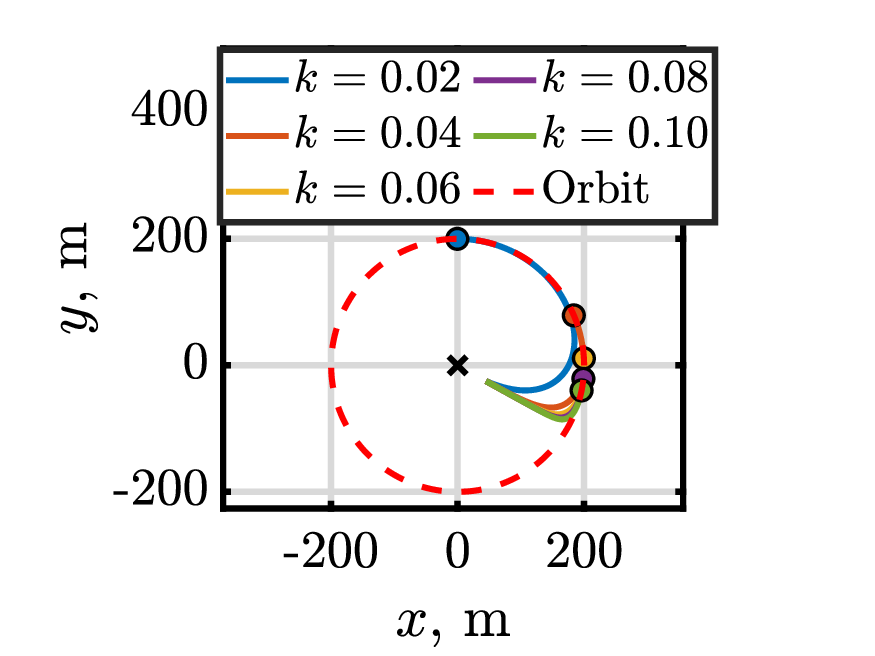}
			\caption{ UAV trajectories}
			\label{fig: UAV trajectories inside ideal}     
		\end{subfigure}%
		\begin{subfigure}[b]{.25\textwidth}
			\centering		\includegraphics[width=\linewidth]{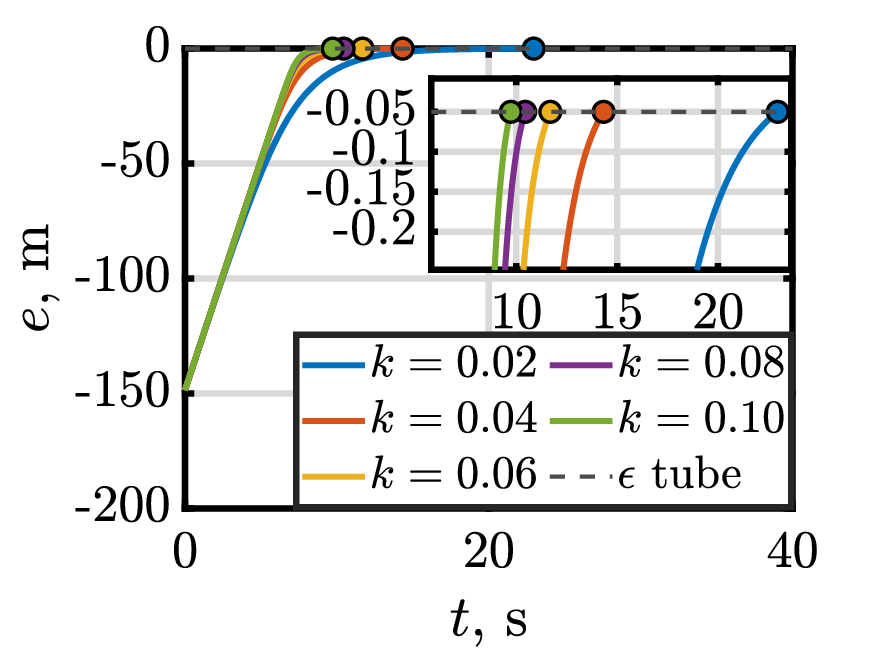}
			\caption{ Standoff error profiles}
			\label{fig: Error profiles inside ideal}
		\end{subfigure}%
        \qquad
        \begin{subfigure}[b]{.25\textwidth}
			\centering			\includegraphics[width=\linewidth]{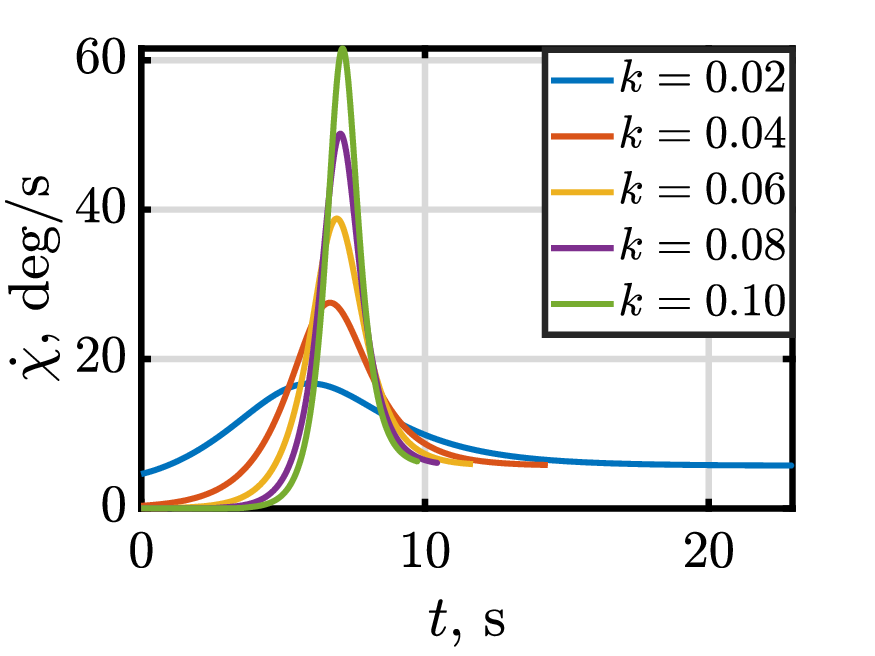}
			\caption{ Turn rate variations}
			\label{fig: Turn rate variations inside ideal}     
		\end{subfigure}%
		\begin{subfigure}[b]{.25\textwidth}
			\centering		\includegraphics[width=\linewidth]{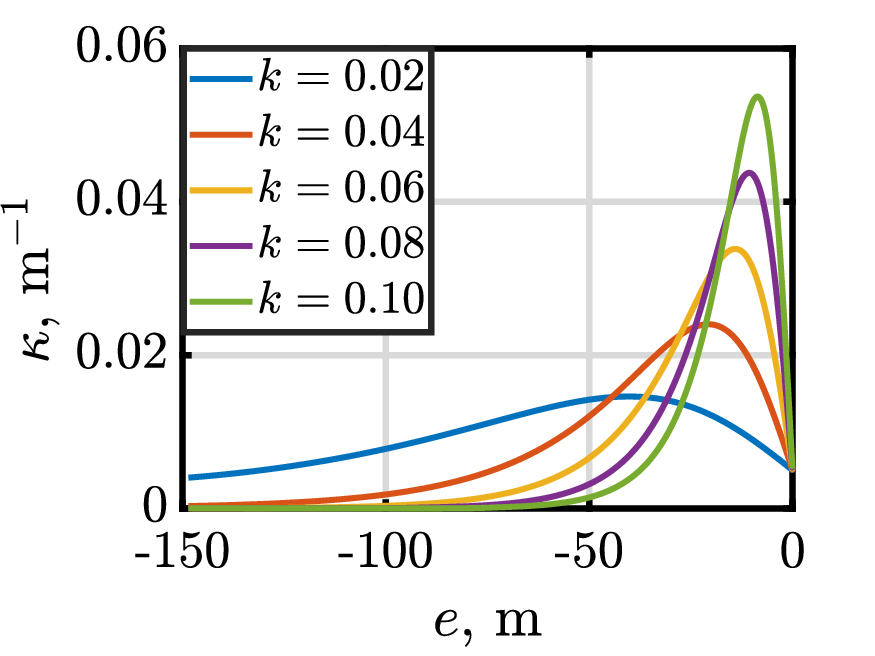}
			\caption{Curvature with standoff error}
			\label{fig: Curvature with standoff error inside ideal}
		\end{subfigure}%
        \caption{ Results for ideal dynamics for inside initial position}
        \label{fig: Results for ideal dynamics for inside initial position}
        \end{figure}
        
        \subsection{Simulation with nonideal initial heading}
        In the second case, the simulation is carried out with a first-order course-angle dynamics governed by
        \begin{align}
        \dot{\chi} = k_\chi(\chi_\mathrm{d} - \chi)
        \label{eq: firs order control}
        \end{align}
where $k_\chi$ denotes the controller gain., The design and control parameters used in the simulation are $(k_G,k_\chi,\dot{\chi}_\mathrm{max}) = (0.02,50,0.5 $rad/s$)$, respectively.  Fig. \ref{fig: Results for nonideal heading simulation for outside initial UAV position nonideal} presents the results for outside initial positions under first-order course control dynamics with both ideal and non-ideal initial heading offsets. Fig. \ref{fig: Trajectories outside nonideal} plots trajectories converging to the desired orbit despite significant heading mismatches (e.g., $-60 ^\circ, + 90 ^\circ$)
 demonstrating the robustness of the GLASS guidance law to initial orientation errors. The corresponding error profiles are plotted in Fig. \ref{fig: error outside nonideal} and exhibit monotonic convergence to the $\epsilon-$tube, with only a slight delay relative to the ideal-heading case due to finite course-loop bandwidth. The turn rate profiles are shown in Fig. \ref{fig: Turn rate profiles outside nonideal} with the corresponding curvature–standoff error relationship plotted in Fig. \ref{fig: Curvature variations with standoff error nonideal}. Fig. \ref{fig: Results for nonideal heading simulation for inside initial UAV position nonideal} shows the corresponding results for inside initial positions. Even with non-ideal initial heading offsets, all trajectories smoothly converge to the desired orbit, as seen in Fig. \ref{fig: Trajectories for inside nonideal}. The standoff error profiles in Fig. \ref{fig: standoff error profiles for inside nonideal} remain monotonically decaying to zero and enter the prescribed 
$\epsilon-$tube, with minor settling-time differences relative to the ideal-heading case. The corresponding turn rate profile and curvature variations with standoff errors are depicted in  Fig. \ref{fig: Turn rate profiles inside nonideal} and \ref{fig: Curvature with standoff error inside nonideal}. These results confirm that the geometric look-angle shaping mechanism maintains stability and bounded curvature demands even when both heading mismatch and finite course dynamics are present. 

         \begin{figure}[ht]
           \begin{subfigure}[b]{.25\textwidth}
			\centering			\includegraphics[width=\linewidth,keepaspectratio]{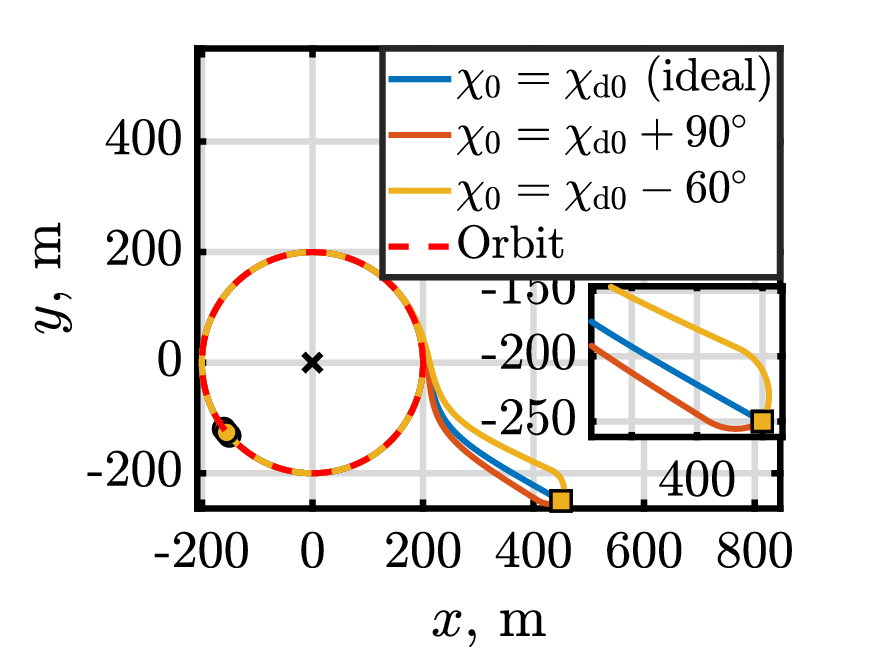}
			\caption{Trajectories }
			\label{fig: Trajectories outside nonideal}     
		\end{subfigure}%
        \begin{subfigure}[b]{.25\textwidth}
        \centering
			\includegraphics[width=\linewidth]{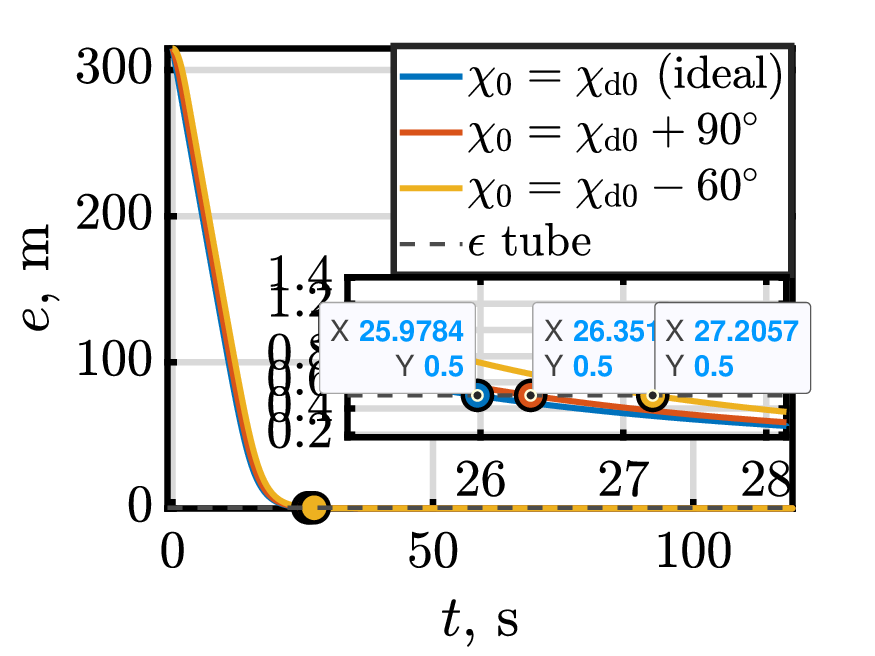}
			\caption{Standoff error profiles }
			\label{fig: error outside nonideal}
		\end{subfigure}%
        \qquad
        \begin{subfigure}[b]{.25\textwidth}
        \centering
			\includegraphics[width=\linewidth]{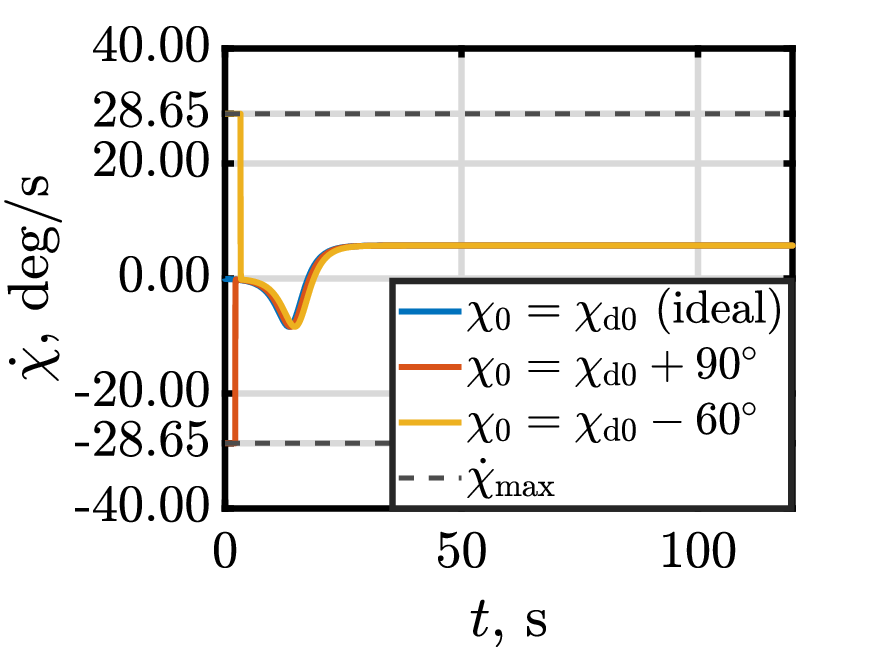}
			\caption{ Turn rate variations }
			\label{fig: Turn rate profiles outside nonideal}
		\end{subfigure}%
        \begin{subfigure}[b]{.25\textwidth}
        \centering		\includegraphics[width=\linewidth]{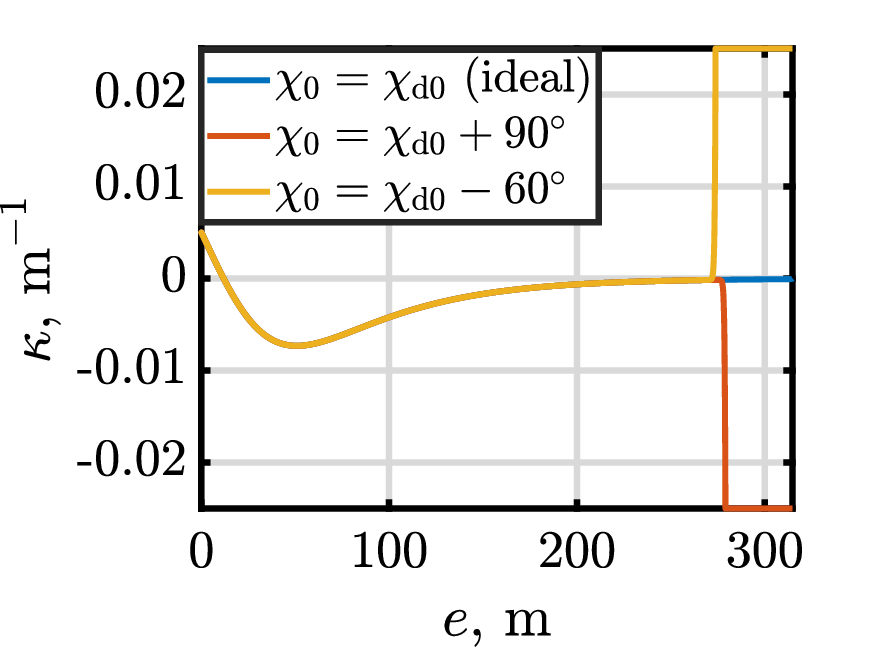}
			\caption{ Curvature with standoff error }
			\label{fig: Curvature variations with standoff error nonideal}
		\end{subfigure}%
        \caption{Results for outside initial UAV position}
        \label{fig: Results for nonideal heading simulation for outside initial UAV position nonideal}
        \end{figure}

        \begin{figure}[ht]
           \begin{subfigure}[b]{.25\textwidth}
			\centering			\includegraphics[width=\linewidth,keepaspectratio]{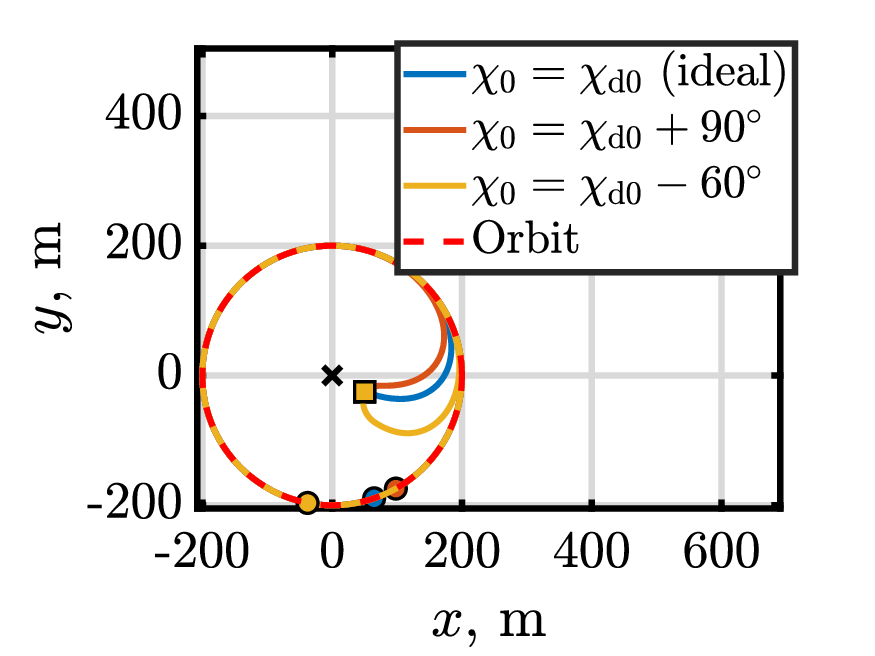}
			\caption{Trajectories}
			\label{fig: Trajectories for inside nonideal}     
		\end{subfigure}%
		\begin{subfigure}[b]{.25\textwidth}
        \centering		\includegraphics[width=\linewidth]{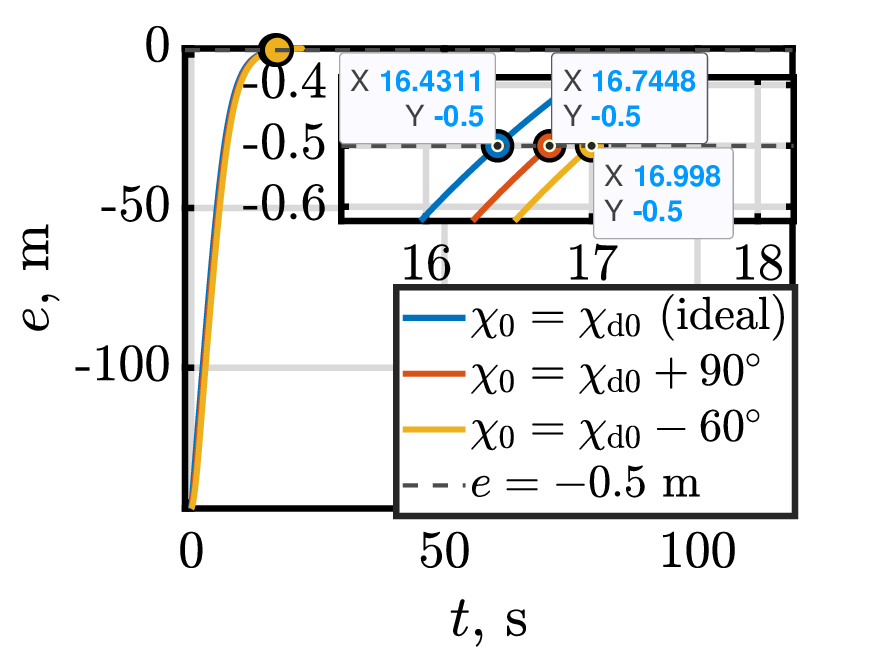}
			\caption{ Standoff error profiles }
			\label{fig: standoff error profiles for inside nonideal}
		\end{subfigure}%
        \qquad
        \begin{subfigure}[b]{.25\textwidth}
        \centering
			\includegraphics[width=\linewidth]{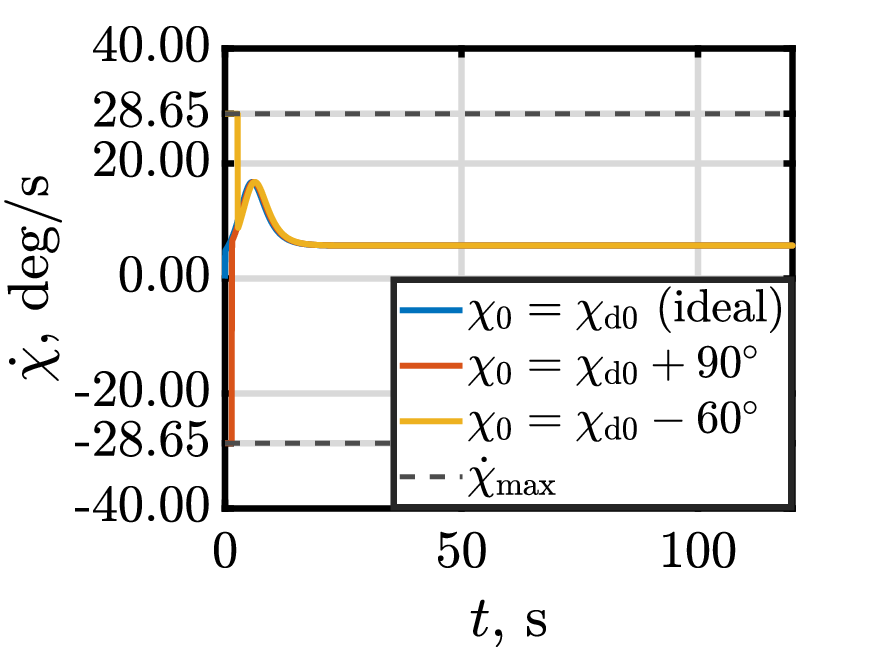}
			\caption{ Turn rate profiles }
			\label{fig: Turn rate profiles inside nonideal}
		\end{subfigure}%
        \begin{subfigure}[b]{.25\textwidth}
        \centering		\includegraphics[width=\linewidth]{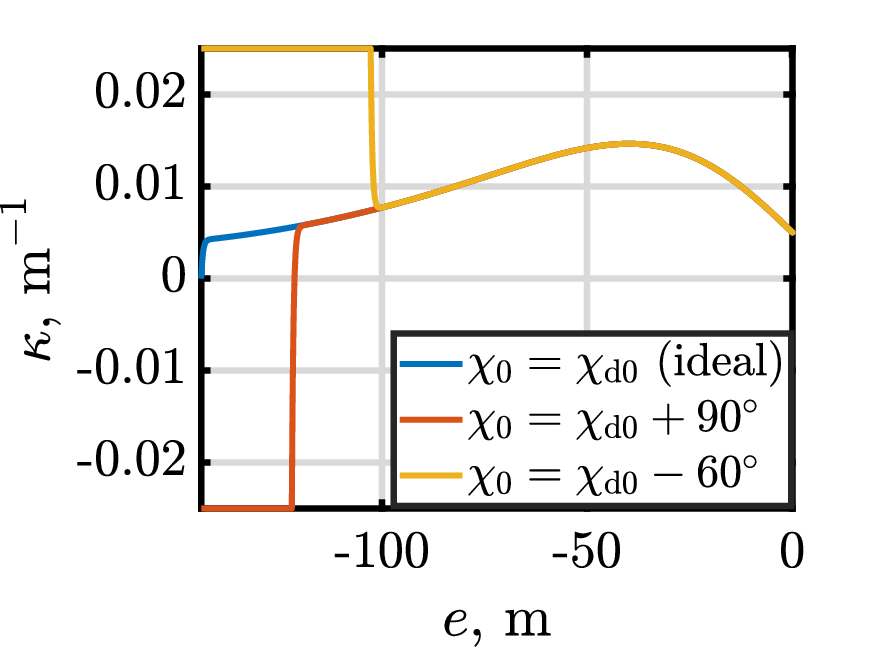}
			\caption{ Curvature with standoff error }
			\label{fig: Curvature with standoff error inside nonideal}
		\end{subfigure}%
        \caption{Results for inside initial UAV position}
        \label{fig: Results for nonideal heading simulation for inside initial UAV position nonideal}
        \end{figure}

Table \ref{tab: Initial conditions and settling times for heading mismatch} summarizes the initial conditions and corresponding analytical and simulated settling times for both outside and inside initial positions under varying heading mismatches. For each case, the analytical settling time $T_\mathrm{ana}$
is computed using the reduced scalar error dynamics \eqref{eq:e_dyn_tanh} whereas the simulated time $T_\mathrm{sim}$
is obtained from full nonlinear simulations including heading dynamics and actuator constraints.
For both outside and inside initial conditions, the analytical settling time remains constant for a given initial radial error, as expected from the closed-form expression which depends solely on $(e_0,k_G,V_g)$
 and is independent of initial heading mismatch. In contrast, the simulated settling time varies with the initial heading angle. This discrepancy arises because the analytical model assumes ideal heading tracking and immediate satisfaction of the algebraic GLASS constraint, whereas the simulation includes transient heading alignment dynamics governed by finite turn-rate limits.
In the outside cases, larger heading mismatches introduce an additional alignment phase before the radial-dominant dynamics become effective, resulting in simulated settling times slightly exceeding the analytical prediction. The deviation remains modest, confirming that the reduced scalar model captures the dominant convergence behavior. Similarly, in the inside cases, the differences between analytical and simulated times are small, indicating that the GLASS guidance rapidly enforces the desired radial contraction once the heading error is corrected.
Overall, the closeness between analytical and simulated settling times validates the theoretical reduction to scalar error dynamics, while the observed differences are primarily attributable to non-ideal heading transients and actuator rate limitations not accounted for in the analytical derivation.

\begin{table}[htbp]
\centering
\caption{Initial conditions and settling times for heading mismatch}
\label{tab: Initial conditions and settling times for heading mismatch}
\begin{tabular}{@{}lccccc@{}}
\toprule
\textbf{Case} & $(x_0, y_0)$, m & $\chi_0$, deg. & $\lambda_0$, deg. & $T_{\mathrm{ana}}$, s & $T_{\mathrm{sim}}$, s \\
\midrule
\multirow{3}{*}{\text{Outside }} & $(450, -250)$ & $150.73$ & $179.79$ & 25.519& 25.
978\\
 & $(450, -250)$ & $-119.27.7$ & $-90.21$ &25.519 & 26.363\\
 & $(450, -250)$ & $90.73$ & $119.79$ & 25.519& 27.199\\
\midrule
\multirow{3}{*}{\text{Inside } } 
 & $(50, -25)$ & $-20.15$ & $6.41$ & 16.977 & 16.442 \\
 & $(50, -25)$ & $69.85$ & $96.41$ & 16.977 & 17.287 \\
 & $(50, -25)$ & $-80.15$ & $-53.59$ & 16.977 & 16.598 \\
\bottomrule
\end{tabular}
\end{table}

\subsection{Comparative study}

To ensure a fair comparison between the proposed GLASS strategy and the arcsine-based
look-angle guidance\cite{jha2024standoff}, the initial commanded heading angle is
matched for both methods. Since the arcsine formulation directly prescribes the
look angle as
\[
\sigma_D = \psi - \gamma = \frac{\pi}{2} + \sin^{-1}(k_D e_D),
\]
where $e_D$ is the normalized range error defined in Eq.~(6) of \cite{jha2024standoff},
equality of the initial heading reduces to enforcing
$\sigma_D(0)=\sigma_G(0)$, where $\sigma_G(0)$ is the LOS-measured look
angle generated by the proposed method. For the circular standoff case, the GLASS constraint yields
\[
\cos \sigma_G = -\tanh\!\left(k_G (d_0 - r_d)\right),
\]
which uniquely determines the initial look angle for a selected orbit
direction. Equating the two look angles gives the explicit gain mapping
\[
k_D = \frac{\sin\!\left(\sigma_G(0) - \frac{\pi}{2}\right)}{e_{0,D}},
\]
where $e_{0,D} = -(1 - r_d/d_0)$ for $d_0 > r_d$.
For the representative case $r_d = 200\,\mathrm{m}$ and
$d_0 = 650\,\mathrm{m}$ (so that $d_0 - r_d = 450\,\mathrm{m}$) with
$k_G = 0.007$, the GLASS shaping term evaluates to
$\tanh(3.15) \approx 0.9963$, yielding
$e_{0,D} = -9/13 \approx -0.6923$ and the matched arcsine gain
$k_D \approx -1.439$.
This guarantees identical initial look angles and hence identical initial
headings for both guidance laws, ensuring that any subsequent differences
in convergence rate, curvature demand, or saturation behavior arise solely
from the intrinsic structure of the guidance design rather than from
mismatched initial conditions, as illustrated by the simulation results presented below. Figure \ref{fig: Comparative results for nonideal heading simulation for outside initial UAV position nonideal} presents comparative simulation results for an outside initial UAV position. As shown in Fig. \ref{fig: Comparative trajectories for outside nonideal}, both guidance laws eventually converge to the desired standoff circle. However, the transient behavior differs markedly. This behavior is consistent with the far-field analysis, wherein the proposed hyperbolic-tangent shaping preserves a significant radial component even for large initial errors. Figure \ref{fig: Comparative standoff error profiles for outside nonideal} compares the evolution of standoff error profiles. The GLASS strategy exhibits faster radial-error decay, reaching the $\epsilon$-tube significantly earlier than the arcsine-based method. The turn-rate profiles in Fig. \ref{fig: Comparative turn rate profiles outside nonideal} demonstrate that the improved convergence speed of GLASS does not come at the cost of excessive control effort. Figure \ref{fig: Comparative Curvature with standoff error outside nonideal} further illustrates the curvature characteristics as a function of the standoff error. The reference method shows larger curvature variations for moderate error values, reflecting its reliance on tangential motion correction.

         \begin{figure}[ht]
           \begin{subfigure}[b]{.25\textwidth}
			\centering			\includegraphics[width=\linewidth,keepaspectratio]{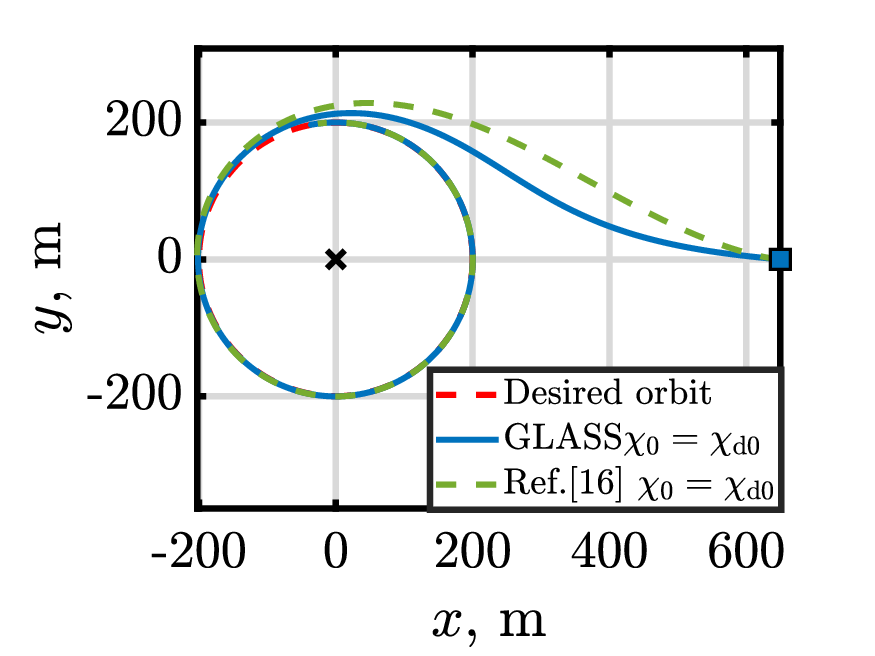}
			\caption{ Trajectories}
			\label{fig: Comparative trajectories for outside nonideal}     
		\end{subfigure}%
		\begin{subfigure}[b]{.25\textwidth}
        \centering		\includegraphics[width=\linewidth]{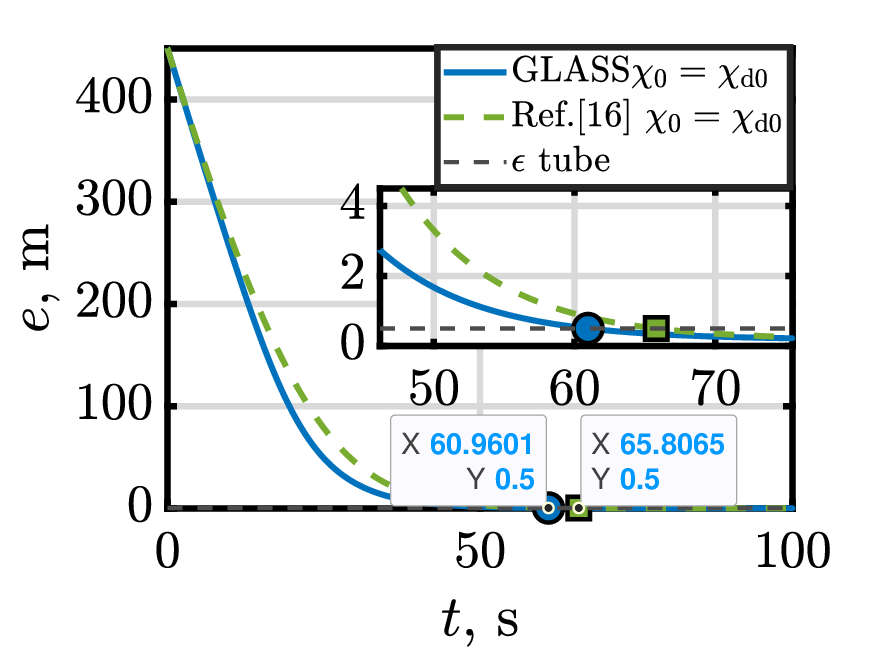}
			\caption{ Standoff error profiles }
			\label{fig: Comparative standoff error profiles for outside nonideal}
		\end{subfigure}%
        \qquad
        \begin{subfigure}[b]{.25\textwidth}
        \centering
			\includegraphics[width=\linewidth]{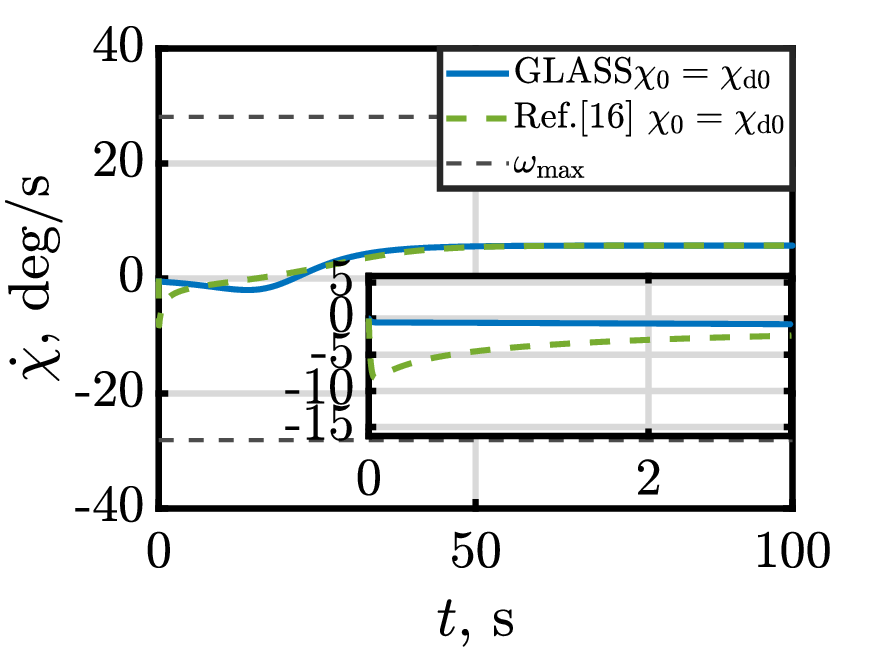}
			\caption{ Turn rate profiles }
			\label{fig: Comparative turn rate profiles outside nonideal}
		\end{subfigure}%
        \begin{subfigure}[b]{.25\textwidth}
        \centering		\includegraphics[width=\linewidth]{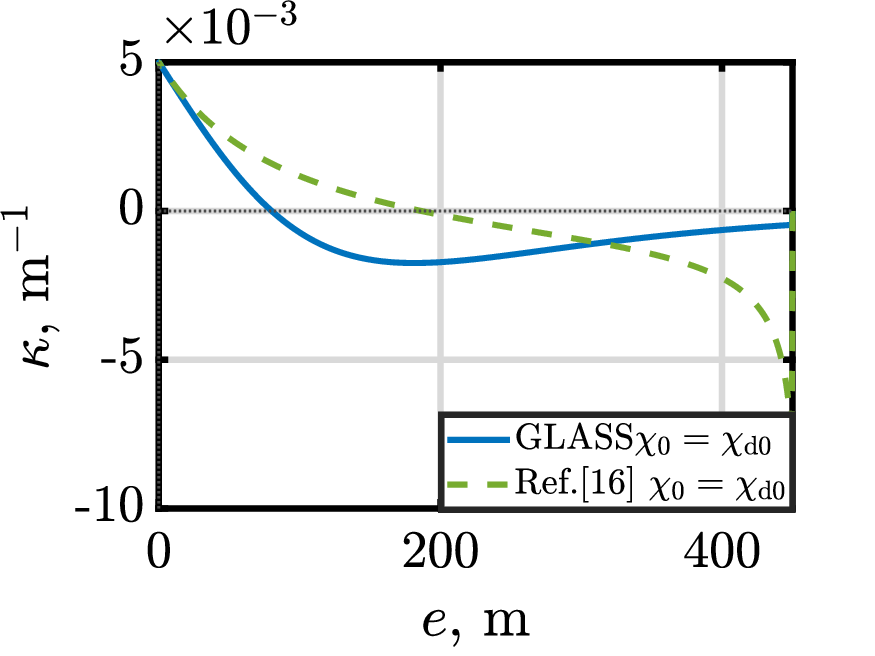}
			\caption{ Curvature with standoff error }
			\label{fig: Comparative Curvature with standoff error outside nonideal}
		\end{subfigure}%
        \caption{Comparative results for outside initial UAV position}
        \label{fig: Comparative results for nonideal heading simulation for outside initial UAV position nonideal}
        \end{figure}


\subsection{6DOF simulation results }

We validate the GLASS guidance scheme-to-6DOF embedding on the dynamic quadrotor model \cite{bouabdallah2004design} to assess the practicality of the proposed method in realistic applications.  
A quadrotor model cannot instantaneously realize $\chi=\chi_\mathrm{d}$. We therefore introduce a first-order course channel with rate saturation:
\begin{equation}
\dot\chi=\mathrm{sat}\left(k_\chi(\chi_\mathrm{f}-\chi),\,\omega_{\max}\right),
\end{equation}
where $\chi_\mathrm{f}$ is a filtered version of $\chi_d$ to mitigate oscillations induced by aggressive shaping gains:
\begin{equation}
\dot\chi_\mathrm{f}=\frac{1}{\tau_{\chi_\mathrm{d}}}\mathrm{wrap}(\chi_\mathrm{d}-\chi_\mathrm{f}).
\end{equation}

\subsubsection{Course-to-acceleration mapping}
Unlike fixed-wing kinematics, the OS4 quadrotor generates planar accelerations through thrust tilting. We convert the (filtered) course state $\chi$ into a horizontal velocity reference:
\begin{equation}
\hat t = \begin{bmatrix}\cos\chi\\ \sin\chi\end{bmatrix},\qquad
\hat r = \begin{bmatrix}\cos\gamma\\ \sin\gamma\end{bmatrix},\qquad
v_{\mathrm{ref}} = V_{\mathrm{ref}}\hat t - k_{\mathrm{rad}}\,e\,\hat r.
\end{equation}
The second term provides explicit radial correction and compensates for inner-loop lag and saturation, thereby reducing steady bias in the achieved orbit radius. A proportional velocity controller generates the commanded horizontal acceleration
\begin{equation}
a_{xy}^{\mathrm{cmd}} = k_v\,(v_{\mathrm{ref}}-v_{xy}),\qquad \|a_{xy}^{\mathrm{cmd}}\|\le a_{\max}.
\end{equation}

\subsubsection{Acceleration-to-attitude mapping and yaw synchronization}
Using the standard small-to-moderate tilt approximation, the commanded inertial accelerations $(a_x,a_y)$ are mapped into desired pitch $\theta_d$ and roll $\phi_d$ commands using the instantaneous yaw $\psi$
\begin{equation}
\theta_d = \frac{a_x\cos\psi+a_y\sin\psi}{g},\qquad
\phi_d   = \frac{a_x\sin\psi-a_y\cos\psi}{g},
\end{equation}
with saturations $|\phi_d|\le \phi_{\max}$ and $|\theta_d|\le \theta_{\max}$. The yaw command is chosen consistent with the course channel,
\begin{equation}
\psi_d=\chi,
\end{equation}
ensuring that the acceleration-to-attitude mapping remains coherent. Ref. \cite{bouabdallah2004design} can be referred for more details.
The simulation parameters are considered as
$m=0.24$ kg, $(I_x,I_y,I_z)=(2.3,2.3,4.0)\times 10^{-3}$ kg$\cdot$m$^2$, $l=0.20$ m, rotor limits $\Omega_i\in[0,600]$ rad/s, and actuator mapping coefficients $(b,d,J_r)=(3.0\times10^{-6},\,1.0\times10^{-7},\,2.0\times10^{-5})$.
The desired standoff radius is $r_d=20$ m around $(x_c,y_c)=(0,0)$ at altitude $z_d=10$ m. The GLASS shaping gain is set to $k_G=0.08$ (CCW), with a non-ideal course channel $(k_\chi,\omega_{\max})=(3.0,0.8)$ and filtering $\tau_{\chi_d}=0.6$ s. The translational outer loop uses $(V_{\mathrm{ref}},k_v,a_{\max},k_{\mathrm{rad}})=(2,2.5,8,0.8)$ and tilt limits $|\phi_d|,|\theta_d|\le 40^\circ$.

Figures \ref{fig: Trajectories outside nonideal 3D} and \ref{fig: Results for 6dof dynamic model} illustrate the 6DOF simulation results, which demonstrate successful embedding of the GLASS kinematic guidance into the full OS4 quadrotor dynamics. As shown in Fig. \ref{fig: Trajectories outside nonideal 3D}, the UAV captures the desired circular standoff orbit from an initial off-orbit condition. It converges to steady circulation at the prescribed altitude, with the planar trajectory closely tracking the desired orbit after a short transient. In Fig. \ref{fig: Results for 6dof dynamic model}, the position tracking plots confirm that $(x(t),y(t))$ converge to the sinusoidal reference generated by the phase-locked orbit, while the altitude $z(t)$ stabilizes rapidly at $z_\mathrm{d} = 10$ m.
The attitude responses $(\phi,\theta,\psi)$ exhibit damped oscillatory transients during the capture phase and then settle smoothly, indicating adequate inner-loop bandwidth relative to the outer-loop guidance dynamics. The control input profiles shows an initial thrust increase to achieve climb and orbit capture, followed by convergence of $U_1$ toward the steady hover-plus-centripetal requirement, while the moment channels $(U_2, U_3, U_4)$ decay to near-zero mean values once steady orbiting is achieved, confirming that sustained turning is realized primarily through constant tilt rather than persistent torque excitation.

     \begin{figure}[ht]
			\centering			
            \includegraphics[width=\linewidth,keepaspectratio]{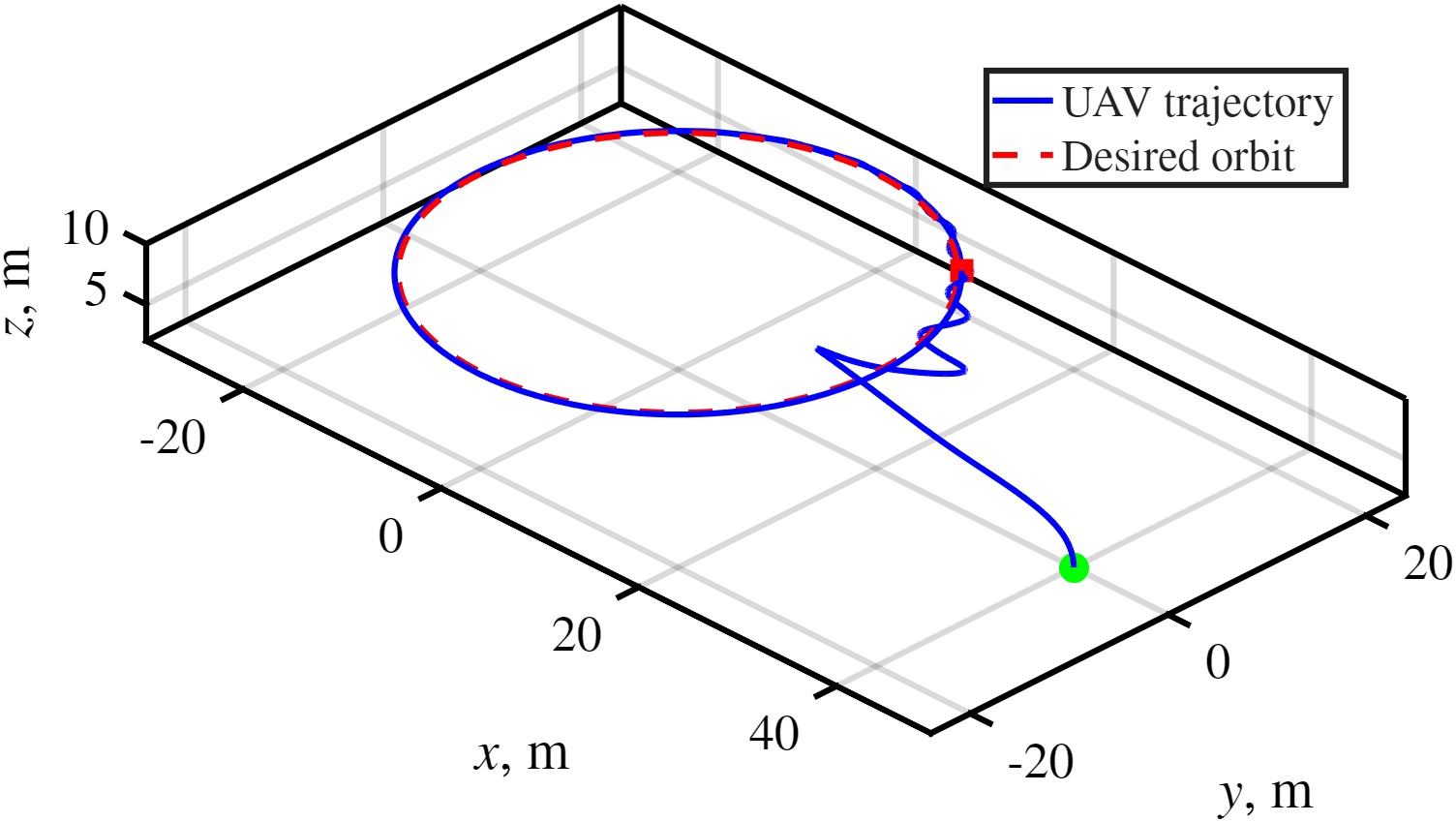}
			\caption{UAV trajectory }
			\label{fig: Trajectories outside nonideal 3D}     
		\end{figure}%

     \begin{figure}[ht]
     \begin{subfigure}[b]{0.5\textwidth}
			\centering			
            \includegraphics[width=0.9\linewidth,keepaspectratio,trim={0 0.25cm 0 1cm},clip]{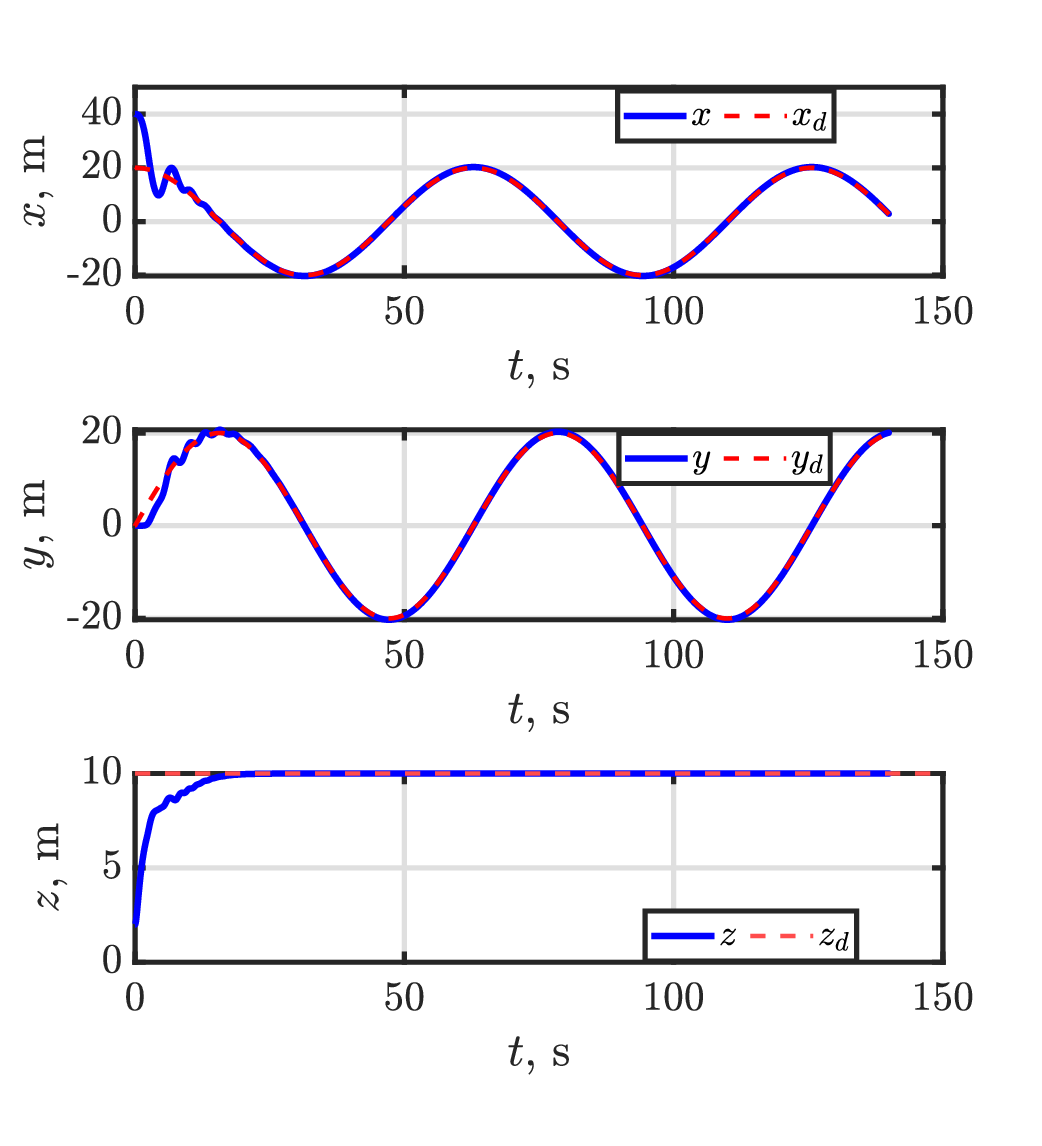}
            \vspace{-0.6cm}
			\caption{Position tracking profiles }
			\label{fig: Position tracking profiles 3D} 
		\end{subfigure}
        \begin{subfigure}[b]{0.5\textwidth}
			\centering	
            \includegraphics[width=0.9\linewidth,keepaspectratio,trim={0 0.25cm 0 1cm},clip]{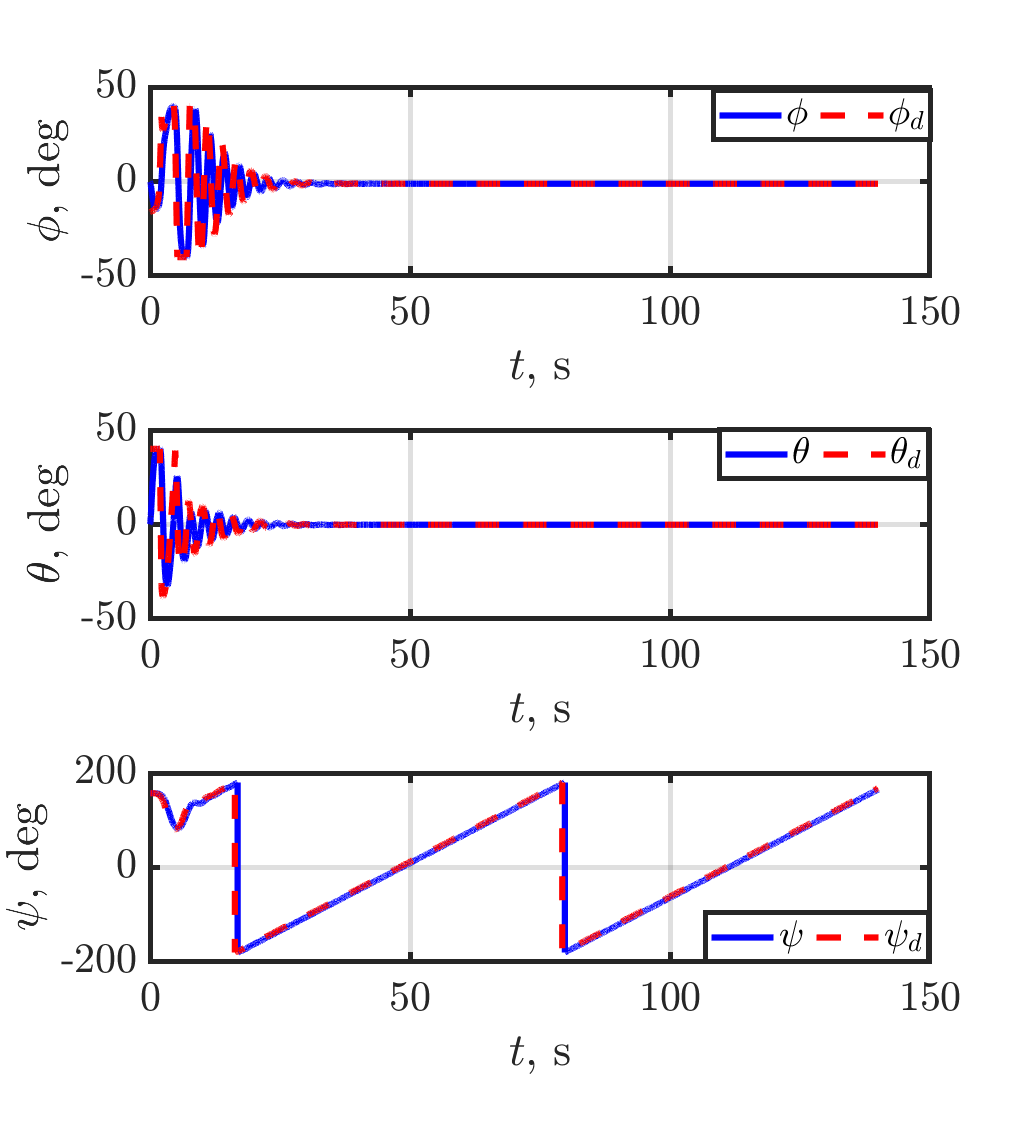}
            \vspace{-0.6cm}
			\caption{Attitude tracking profiles }
			\label{fig: Attitude tracking profiles 3D}  
		\end{subfigure}
      
        \begin{subfigure}[b]{0.25\textwidth}
			\centering			
            \includegraphics[width=0.9\linewidth,keepaspectratio,trim={0 0.25cm 0 0.25cm},clip]{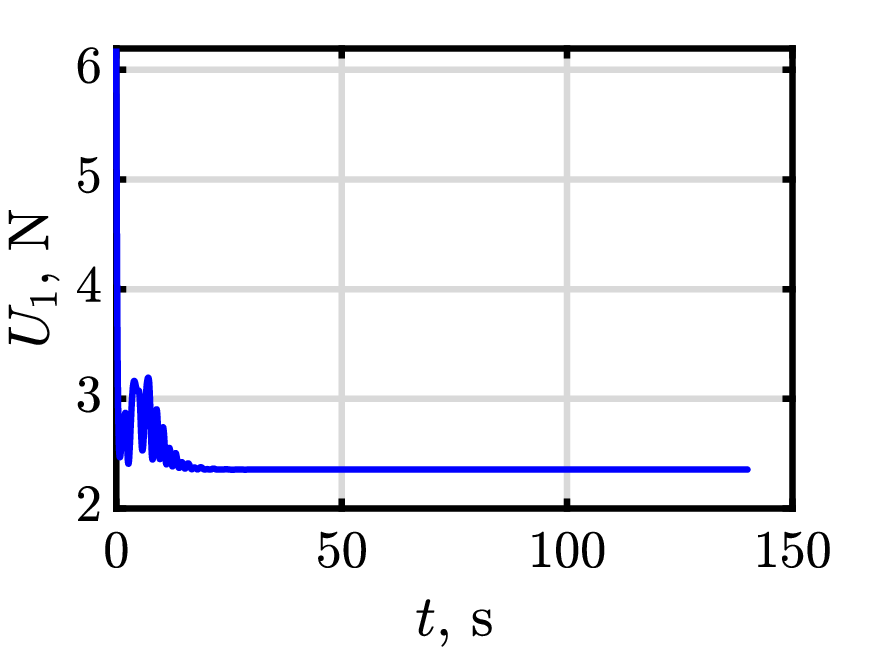}
			\caption{ Thrust profile}
			\label{fig: Thrust profile}     
		\end{subfigure}%
		\begin{subfigure}[b]{0.25\textwidth}
        \centering		\includegraphics[width=0.9\linewidth,trim={0 0.25cm 0 0.25cm},clip]{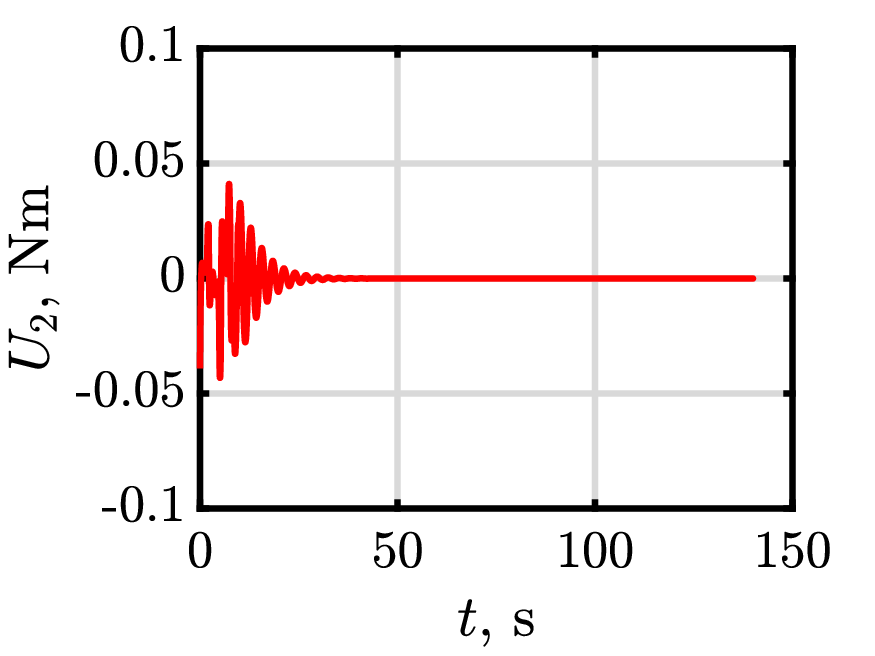}
			\caption{ Roll profile }
			\label{fig: Roll profile}
		\end{subfigure} 
        \qquad
        \begin{subfigure}[b]{.25\textwidth}
        \centering
			\includegraphics[width=0.9\linewidth,trim={0 0.25cm 0 0.25cm},clip]{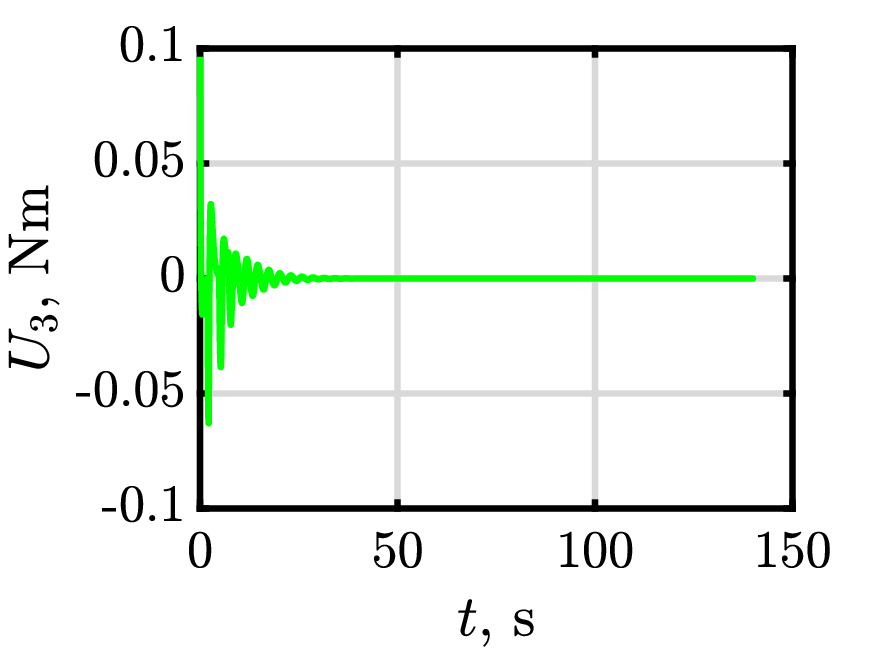}
			\caption{ Pitch profile }
			\label{fig: Pitch profile}
		\end{subfigure}%
        \begin{subfigure}[b]{.25\textwidth}
        \centering		\includegraphics[width=0.9\linewidth,trim={0 0.25cm 0 0.25cm},clip]{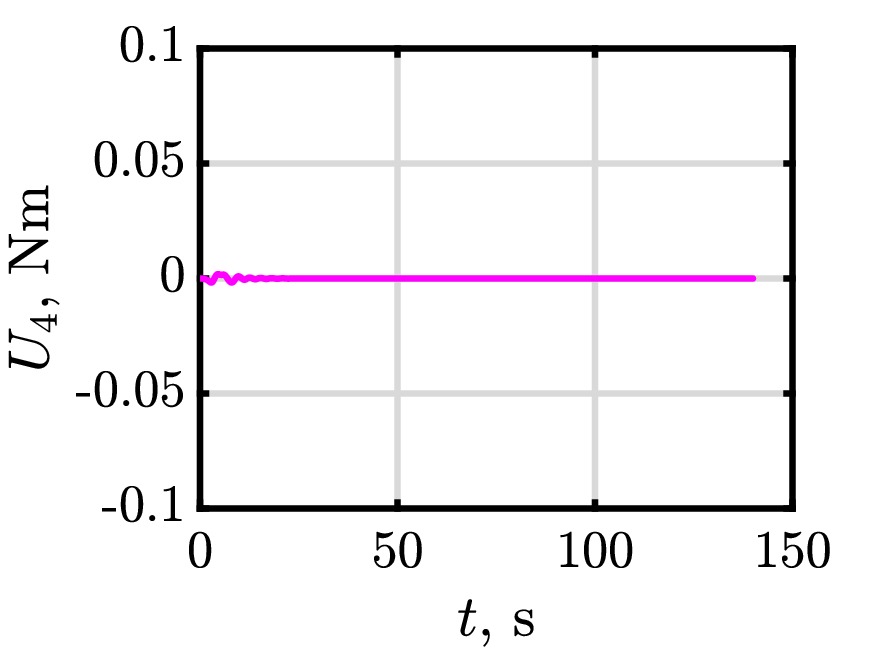}
			\caption{ Yaw profile }
			\label{fig: Yaw profile}
		\end{subfigure}%
        \caption{Results for 6dof dynamic model}
        \label{fig: Results for 6dof dynamic model}
\end{figure}

\section{Conclusions}
This paper presented a geometry-driven look-angle shaping framework for standoff target tracking under turn-rate constraints. In contrast to gain-amplified shaping strategies, where admissible convergence speed is limited by initial-condition-dependent gain bounds, the proposed approach regulates radial contraction through bounded nonlinear evolution of the look angle itself. The engagement dynamics were formulated in polar coordinates, enabling analytical settling-time characterization under radial-dominant conditions. Lyapunov analysis established global asymptotic convergence to the prescribed standoff circle while embedding curvature feasibility directly within the guidance design. Simulation results demonstrated faster far-field capture without excessive turn-rate demand, thereby expanding the admissible operational envelope for practical surveillance and inspection missions.

Future work will address moving targets and three-dimensional extensions, with outdoor experimental tests.


\bibliographystyle{ieeetr}
\bibliography{refICUAS.bib}
\end{document}